\documentclass{article}

\usepackage{times}
\usepackage{graphicx} 
\usepackage{subfigure}

\usepackage{natbib}

\usepackage{color}
\usepackage{amssymb,amsfonts}
\usepackage{amsmath,amssymb}
\usepackage{float}
\usepackage{comment}
\usepackage{verbatim}
\usepackage{appendix,stfloats}
\usepackage{bbm}
\usepackage{blkarray}
\usepackage{verbatim}
\usepackage{algorithm}
\usepackage{multirow}
\usepackage[noend]{algpseudocode}
\makeatletter
\def\BState{\State\hskip-\ALG@thistlm}
\makeatother

\usepackage{hyperref}

\newtheorem{theorem}{Theorem}[section]
\newtheorem{lemma}[theorem]{Lemma}

\newtheorem{corollary}[theorem]{Corollary}

\newenvironment{proof}[1][Proof]{\begin{trivlist}
\item[\hskip \labelsep {\bfseries #1}]}{\end{trivlist}}

\newcommand{\indicator}[1]{\mathds{1}_{\left[ {#1} \right] }}
\usepackage{array,dsfont}
\def \bx{\mathbf x}
\def \cT{\mathcal{T}}
\usepackage{tikz,amsmath}
\usepackage{tikz-qtree}
\usetikzlibrary{calc}
\usetikzlibrary{arrows}
\usepackage{verbatim}
\usetikzlibrary{shapes.geometric}

\usepackage{kbordermatrix}
\usepackage[accepted]{icml2015}

\icmltitlerunning{Optimally Pruning Decision Tree Ensembles With Feature Cost}
\makeatletter
\def\ps@pprintTitle{%
 \let\@oddhead\@empty
 \let\@evenhead\@empty
 \def\@oddfoot{}%
 \let\@evenfoot\@oddfoot}
\makeatother

\begin{document}

\twocolumn[
\icmltitle{Optimally Pruning Decision Tree Ensembles With Feature Cost}

\icmlauthor{Feng Nan}{fnan@bu.edu}
\icmlauthor{Joseph Wang}{joewang@bu.edu}
\icmlauthor{Venkatesh Saligrama}{srv@bu.edu}
\icmladdress{Boston University, 8 Saint Mary's Street, Boston, MA}

\vskip 0.3in
]

\begin{abstract}
We consider the problem of learning decision rules for prediction with feature budget constraint. In particular, we are interested in pruning an ensemble of decision trees to reduce expected feature cost while maintaining high prediction accuracy for any test example. We propose a novel 0-1 integer program formulation for ensemble pruning. Our pruning formulation is general - it takes any ensemble of decision trees as input. By explicitly accounting for feature-sharing across trees together with accuracy/cost trade-off, our method is able to significantly reduce feature cost by pruning subtrees that introduce more loss in terms of feature cost than benefit in terms of prediction accuracy gain. Theoretically, we prove that a linear programming relaxation produces the exact solution of the original integer program. This allows us to use efficient convex optimization tools to obtain an optimally pruned ensemble for any given budget. Empirically, we see that our pruning algorithm significantly improves the performance of the state of the art ensemble method BudgetRF.
\end{abstract}

\section{Introduction}
Many modern applications of supervised machine learning face the challenge of test-time budget constraints. For example, in internet search engines \cite{YahooChallenge2010}, features of the query-document pair are extracted whenever a user enters a query at the cost of some CPU time in order to rank the relevant documents. The ranking has to be done in milliseconds to be displayed to the user, making it impossible to extract computationally expensive features for all documents. Rather than simply excluding these computationally expensive features, an adaptive decision rule is needed, so that only cheap features are extracted for the majority of queries and expensive features are extracted for only a small number of difficult queries. Many approaches have been proposed by various authors to solve such test-time budget constraint problem \cite{Gao+Koller:NIPS11,DBLP:conf/icml/XuWC12,trapeznikov:2013b,wang2014lp,wang2014model,Fastmarginbasedcostsensitiveclassification_NanWTS14,NIPS2015_5982}. 

Nan et al. \cite{icml2015_nan15} proposed a novel random forest approach for test-time feature cost reduction. During training, an ensemble of decision trees are built based on random subsampling the training data for each decision tree. A class of \emph{admissible} (essentially monotone and supermodular) impurity functions together with the cost of each feature are used to greedily determine the data split at each internal node of the decision trees. During prediction, a test example is run through each of the trees in the ensemble and the majority label is assigned to the test example. Such a simple strategy is shown to yield a worst-case cost at most $O\left(\log(n)\right)$ times the optimal cost for each decision tree built on $n$ training samples. Empirically, it is shown to have state-of-the-art performance in terms of prediction-cost tradeoff.

The trees in these budgeted random forests are built independently, ignoring the fact that repeated use of the same feature does not incur repeated feature acquisition cost. We exploit interdependencies among the ensemble of trees to achieve better accuracy - cost tradeoff. Theoretically, we propose a general ensemble pruning formulation that solves the accuracy-cost tradeoff exactly; empirically, we demonstrate significant improvement. 

The focus of this paper is on pruning ensembles of decision trees. We assume an ensemble of decision trees are given as inputs; such an ensemble can be obtained using the algorithm proposed by Nan et al. \cite{icml2015_nan15} or any other decision tree ensemble method. Our main contribution is the development of an efficient algorithm for pruning an ensemble of decision trees to explicitly tradeoff prediction accuracy and feature cost. 


\section{Related Work}
Although decision tree pruning has been studied extensively to improve generalization performance, we are not aware of any existing pruning method that takes into account the feature costs.

A popular heuristic for pruning to reduce generalization error is Cost-Complexity Pruning (CCP), introduced by Breiman et al. \cite{breiman1984classification}. It defines a \emph{cost-complexity} measure for each subtree of the decision tree as sum of two terms: the number of misclassified examples in the subtree plus the number of leaves in the subtree times a tradeoff parameter. This measure is also computed when the subtree is pruned to become a leaf. As the tradeoff parameter increases, more emphasis is given to reducing the size of the subtree compared to minizing the number of misclassified examples. The CCP algorithm iteratively selects the subtree with the lowest cost-complexity measure if it were pruned as the tradeoff parameter gradually increases. At each iteration the selected subtree is pruned and the cost-complexity measures are re-computed for the next iteration. Each pruned tree produced in this procedure is optimal with respect to size - no other subtree of the same number of leaves would have a lower misclassification rate than the one obtained by this procedure. As pointed out by Li et al. \cite{Li2001DPP}, CCP has undesirable ``jumps" in the sequence of pruned tree sizes. To alleviate this, they proposed a Dynamic-Program-based Pruning (DPP) method for binary trees. The DPP algorithm is able to obtain optimally pruned trees of all sizes, however, faces the curse of dimensionality when pruning an ensemble of decision trees and taking  feature cost into account. 

Generally, pruning is not considered when constructing random forests as overfitting is avoided by constructing an ensemble of trees. The ensemble approach is a strong approach to avoiding overfitting, however test-time budget constraint problems require consideration of both cost and accuracy.
 
Kulkarni and Sinha \cite{pruningRFSurvey} provide a survey of methods to prune random forests in order to reduce ensemble size. However, these methods do not explicitly account for feature costs.

\section{Background and Notations} A training sample $S=\{(\bx^{(i)},y^{(i)}):{i=1,\dots,N}\}$ is generated i.i.d. from an unknown distribution,  where $\bx^{(i)} \in \Re^K$ is the feature vector with a cost assigned to each of the $K$ features and $y^{(i)}$ is the label for the $i$th example. In the case of multi-class classification $y \in \{1,\dots,M\}$, where $M$ is the number of classes. Given a decision tree $\mathcal{T}$, we index the nodes as $h\in \{1,\dots,|\cT|\}$, where node $1$ represents the root node. For any $h\in \cT$, we define the following standard terminology:

$p(h) \equiv$ set of predecessor nodes of $h$ $\equiv$ set of nodes (excluding $h$) that lie on the path from the root node to $h$. 

$\mathcal{T}_h \equiv$ subtree of $\mathcal{T}$ that is rooted at node $h$. 

$\tilde{\cT} \equiv$ set of leaf nodes of tree $\cT$.


$b(h) \equiv$ set of brother (sibling) nodes of $h \equiv$ set of nodes who share the same immediate parent node as $h$. 

$S_h \equiv$ the set of examples in $S$ routed to or through $h$ on $\mathcal{T}$.

$\text{Pred}_h \equiv$ predicted label at node $h$ on $\mathcal{T}$ based on the class distribution of $S_h$. It is equal to the class with the most number of training examples at $h$.

$e_h \equiv$ number of misclassified examples in $S_h$ based on $\text{Pred}_h$. It is equal to $\sum_{i\in S_h} \indicator{y^{(i)}\neq \text{Pred}_h}$.

Finally, the corresponding definitions for $\mathcal{T}$ can be extended to an ensemble of $T$  decision trees $\{\mathcal{T}_t :t=1,\dots,T\}$ by adding an subscript $t$. 

The process of pruning $\mathcal{T}$ at $h$ involves \emph{collapsing} $\mathcal{T}_h$ and making $h$ a leaf node. We say a pruned tree $\mathcal{T}_{\text{Prune}}$, having $\tilde{\cT}_{\text{Prune}}$ as its set of leaf nodes, is a \emph{valid} pruned tree of $\mathcal{T}$ if (1) $\mathcal{T}_{\text{Prune}}$ is a subtree of $\mathcal{T}$ containing root node 1 and (2) for any $h\neq 1$ contained in $\mathcal{T}_{\text{Prune}}$, the sibling nodes $b(h)$ must also be contained in $\mathcal{T}_{\text{Prune}}$.

For a given tree $\mathcal{T}$, let us define the following binary variable for each node $h\in \cT$
$$
z_h=\left\{
\begin{array}{rl}
1 & \text{if node } h \text{ is a leaf in the pruned tree} ,\\
0 & \text{otherwise}.
\end{array} \right.
$$
Proposition 1 of \cite{OptimalPruning2009} showed that the following set of constraints completely characterize the set of valid pruned trees of $\mathcal{T}$. 
$$
z_h+\sum_{u\in p(h)} z_u=1 \qquad \forall h \in \tilde{\mathcal{T}}, \\
$$
$$
z_h \in \{0,1\} \qquad \forall h\in \cT.
$$

A common decision tree pruning objective is to keep the probability of prediction error in the pruned tree as low as possible while reducing the number of tree nodes. Given a decision tree $\cT$, it is easy to see that the overall probability of prediction error is of the pruned tree $\mathcal{T}_{\text{Prune}}$ is
\begin{equation}
\frac{1}{N}\sum_{h\in \cT} e_h z_h. \label{eq:errObj}
\end{equation}

Therefore a decision tree pruning problem can be formulated as the following integer program
\begin{equation*}
\begin{array}{rlcl}
\displaystyle \min_{z_h} & \multicolumn{2}{l}{\frac{1}{N}\sum_{h\in \cT} e_h z_h} \\
\textrm{s.t.} & z_h+\sum_{u\in p(h)} z_u=1 &\forall h \in \tilde{\mathcal{T}}, \\
& z_h \in \{0,1\}  &\forall h\in \mathcal{T}.
\end{array}\tag{IP0}\label{eq:IP0}
\end{equation*}
By showing that the constraint matrix can be turned into a network matrix form, \cite{OptimalPruning2009} showed the above integer problem can be solved exactly by linear program relaxation. 

\section{Pruning with Feature Costs} 
Suppose the feature costs are given by $\{c_k:k=1,\dots,K\}$.
The feature cost incurred by an example is the total costs of \emph{unique} features it encounters in all trees. This is because we assume whenever a feature is acquired its value is cached and subsequent usage incurs no additional cost. 
Specifically, the cost of classifying an example $i$ on decision tree $\mathcal{T}$ is given by 
\begin{equation*}
c(\cT,\bx^{(i)})=\sum_{k=1}^{K}c_k \indicator{\text{feature } k \text{ is used by } \bx^{(i)} \text{ in } \cT}=\sum_{k=1}^{K}c_k w_{k,i},
\end{equation*}
where the binary variables $w_{k,i}$ serve as the indicator variables:
$$
w_{k,i}=\left\{\begin{array}{rl}
1 & \text{ if feature } k \text{ is used by }\bx^{(i)} \text{ in } \cT,\\
0 & \text{ otherwise}.
\end{array} \right.
$$

Similarly, the cost of classifying $\bx^{(i)}$ on an ensemble of $T$ trees is 
\begin{equation*}
c(\cT_{[T]},\bx^{(i)})=\sum_{k=1}^{K}c_k \indicator{\text{feature } k \text{ is used by } \bx^{(i)} \text{ in any } \cT_t, t=1,\dots,T}. 
\end{equation*}

%

In a pruned tree $\cT_{\text{Prune}}$ we can encode the conditions for $w_{k,i}$'s using the leaf indicator variable $z_h$'s. If $z_h=1$ for some node $h$, then the examples that are routed to $h$ must have used all the features in the predecessor nodes $p(h)$. We use $k\sim p(h)$ to denote feature $k$ is used in any predecessor of $h$. Then for each feature $k$ and example $i$, we must have $w_{k,i}\geq z_h$ for all nodes $h$ such that $i\in S_h$ and $k\sim p(h)$. Combining the error term \eqref{eq:errObj} and feature cost in the objective, we arrive at the following integer program:
\begin{equation*}
\begin{array}{rlll}
\displaystyle \min_{z_h,w_{k,i}} & \multicolumn{2}{l}{\frac{1}{N}\displaystyle  \sum_{h\in \mathcal{N}} e_h z_h +\lambda \sum_{k=1}^{K}c_k(\frac{1}{N}\sum_{i=1}^{N}w_{k,i})} \\
\textrm{s.t.} & z_h+ \sum_{u\in p(h)} z_u=1 & \forall h \in \tilde{\mathcal{T}}, \\
& z_h \in \{0,1\} & \forall h\in \mathcal{T},  \\
& w_{k,i}\geq z_h & \forall h:i\in S_h \land k\sim p(h), \\
& & \forall k\in [K], \forall i\in S,   \\
& w_{k,i} \in \{0,1\} & \forall k\in [K], \forall i\in S.
\end{array}
\tag{IP1}\label{eq:IP1}
\end{equation*}
Again, the constraint $w_{k,i}\geq z_h$ ensures that if $h$ is a leaf node in the pruned tree ($z_h=1$) and the $i$th example encounters feature $k$ along the way before arriving at $h$ then $w_{k,i}$ must be 1. 

\tikzset{every tree node/.style={minimum width=1em,draw,circle},
         blank/.style={draw=none},
         edge from parent/.style=
         {draw, edge from parent path={(\tikzparentnode) -- (\tikzchildnode)}},
         level distance=1cm}

Unfortunately, unlike \eqref{eq:IP0}, the constraint set in \eqref{eq:IP1} has fractional extreme points, leading to possibly fractional solutions to the relaxed problem. Consider Tree 1 in Figure \ref{fig:trees}. Feature 1 is used at the root node and feature 2 is used at node 3. There are 7 variables (assuming there is only one example and it goes to leaf 4):
$$
z_1,z_2,z_3,z_4,z_5,w_{1,1},w_{2,1}.
$$

The LP relaxed constraints are:
\begin{align*}
& z_1+z_3+z_4=1 , z_1+z_3+z_5=1 , z_1+z_2=1, \\
& w_{1,1}\geq z_4, w_{1,1}\geq z_3, w_{2,1}\geq z_4, 0\leq z\leq 1.
\end{align*}

The following is a basic feasible solution:
\begin{equation*}
z_1=0,  z_2=1 , z_3=z_4=z_5=0.5, w_{1,1}=w_{2,1}=0.5,
\end{equation*}
because the following set of 7 constraints are active:
\begin{align*}
& z_1+z_3+z_4=1, z_1+z_3+z_5=1, \\
& w_{1,1}\geq z_4, w_{1,1}\geq z_3,w_{2,1}\geq z_4, z_1=0,z_2=1.
\end{align*}

Even if we were to interpret the fractional solution of $z_h$ as probabilities of $h$ being a leaf node, we see an issue with this formulation: the example has $0.5$ probability of stopping at node 3 or 4 ($z_3=z_4=0.5$). In both cases feature 1 at the root node has to be used; but $w_{1,1}=0.5$ indicates that it's only being used half of the times, which is undesirable at all. 

We have seen the LP relaxation of \eqref{eq:IP1} fails to capture the desired behavior of the integer program. We now examine an alternative formulation and show that the optimal solution of its LP relaxation is exactly that of the integer program.

Given a tree $\cT$, feature $k$ and example $\bx^{(i)}$, let $u_{k,i}$ be the first node associated with feature $k$ on the root-to-leaf path the example follows in $\cT$. Clearly, feature $k$ is used by  $\bx^{(i)}$ if and only if none of the nodes between root and $u_{k,i}$ is leaf.
In terms of constraints, we have
\begin{equation}
w_{k,i}+z_{u_{k,i}}+\sum_{h\in p(u_{k,i})} z_h = 1
\end{equation}
as long as feature $k$ is used by $\bx^{(i)}$ in $\cT$. Intuitively, this constraint ensures that for the binary variable $w_{k,i}$ to be non-zero, the tree cannot be pruned before the feature $k$ is obtained (the summation in the constraint equal to zero) and  the feature $k$ must be used in order to split the data (the term $z_{u_{k,i}}$ in the constraint equal to zero).

For a given tree $\cT$ we arrive at the following formulation.
\begin{equation*}
\begin{array}{rll}
\displaystyle \min_{z_h,w_{k,i}} & \multicolumn{1}{l}{\frac{1}{N}\displaystyle  \sum_{h\in \mathcal{N}} e_h z_h +\lambda \sum_{k=1}^{K}c_k(\frac{1}{N}\sum_{i=1}^{N}w_{k,i})} \\
\textrm{s.t.} & z_h+ \sum_{u\in p(h)} z_u=1 \hspace{41 pt} \forall h \in \tilde{\mathcal{T}}, \\
& z_h \in \{0,1\} \hspace{82 pt} \forall h\in \mathcal{T},  \\
& \displaystyle w_{k,i}+z_{u_{k,i}}+\sum_{h\in p(u_{k,i})} z_h =1, \forall k\in K_i,\forall i\in S,  \\
& w_{k,i} \in \{0,1\} \hspace{73 pt} \forall k\in [K], \forall i\in S,
\end{array}
\tag{IP2}\label{eq:IP2}
\end{equation*}
where $K_i$ denotes the set of features the $i$th example uses on tree $\cT$. 


\paragraph{From tree to ensemble: } we generalize \eqref{eq:IP2} to ensemble pruning with tree index $t$: $z^{(t)}_h$ indicates whether node $h$ in $\cT_t$ is a leaf; $w^{(t)}_{k,i}$ indicates whether feature $k$ is used by the $i$th example in $\cT_t$; $w_{k,i}$ indicates whether feature $k$ is used by the $i$th example in any of the $T$ trees $\cT_1,\dots,\cT_T$; $u_{t,k,i}$ is the first node that associated with feature $k$ on the root-to-leaf path the example follows in $\cT_t$. Note that we minimize the average empirical probability of error across all trees, which corresponds to the error of prediction based on averaging the leaf distributions across the ensemble for a given example. 
\begin{equation*}
\begin{array}{rll}
\displaystyle \min_{z^{(t)}_h,w^{(t)}_{k,i}} & \multicolumn{1}{l}{\frac{1}{NT}\displaystyle  \sum_{t=1}^{T}\sum_{h\in \mathcal{N}^{(t)}} e^{(t)}_h z^{(t)}_h +\lambda \sum_{k=1}^{K}c_k(\frac{1}{N}\sum_{i=1}^{N}w_{k,i})} \\
\textrm{s.t.} & z^{(t)}_h+ \sum_{u\in p(h)} z^{(t)}_u=1 \hspace{16 pt} \forall h \in \tilde{\mathcal{T}}_t, \forall t \in [T],\\
& z^{(t)}_h \in \{0,1\} \hspace{63 pt} \forall h\in \mathcal{T}_t, \forall t\in [T], \\
& \displaystyle w^{(t)}_{k,i}+ z^{(t)}_{u_{t,k,i}}+\sum_{h\in p(u_{t,k,i})} z^{(t)}_h=1 , \\
& \hspace{72 pt} \forall k\in K_{t,i},\forall i\in S, \forall t \in [T],\\
& w^{(t)}_{k,i} \in \{0,1\} \hspace{26 pt} \forall k\in [K], \forall i\in S \forall t\in [T],\\
& w^{(t)}_{k,i} \leq w_{k,i} \hspace{26 pt} \forall k\in [K], \forall i\in S \forall t\in [T],\\
& w_{k,i}\in \{0,1\} \hspace{56 pt} \forall k\in [K], \forall i\in S.
\end{array}
\tag{IP3}\label{eq:IP3}
\end{equation*}

\begin{lemma}\label{lemma1}
The equality constraints in \eqref{eq:IP3} can be turned into an equivalent network matrix form for each tree.
\end{lemma}
\begin{proof}
This is simply due to an observation that $w^{(t)}_{k,i}$ can be regarded as just another $z$ variable for a fictitious child node of $u_{t,k,i}$ and the rest of proof follows directly from the construction in Proposition 3 of \cite{OptimalPruning2009}.
\end{proof}

\tikzset{every tree node/.style={minimum width=1em,draw,circle},
         blank/.style={draw=none},
         edge from parent/.style=
         {draw, edge from parent path={(\tikzparentnode) -- (\tikzchildnode)}},
         level distance=1cm}
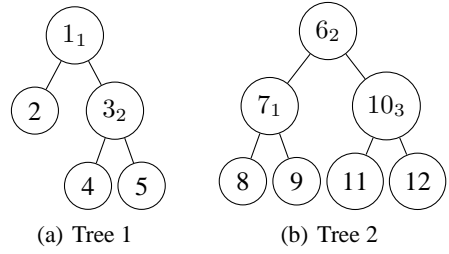
\begin{figure}
\centering
\subfigure[Tree 1]{
\begin{tikzpicture}
\Tree
[.\ensuremath{1_1} 
    [.2 ]    
    [.\ensuremath{3_2} 
    	[.4 ]
    	[.5 ]
    ]
]
\end{tikzpicture}}
~
~
~
\subfigure[Tree 2]{
\begin{tikzpicture}
\Tree
[.\ensuremath{6_2} [.\ensuremath{7_1} [.8 ] [.9 ] ]  [.\ensuremath{10_3} [.11 ] [.12 ] ] ]
\end{tikzpicture}}
\caption{An ensemble of two decision trees with node numbers and associated feature in subscripts}\label{fig:trees}
\end{figure}
Figure \ref{fig:trees} illustrate such a construction. For simplicity we consider only one example being routed to nodes 4 and 11 respectively on the two trees. The equality constraints in \eqref{eq:IP3} can be separated based on the trees and put in matrix form:
\renewcommand{\kbldelim}{(}
\renewcommand{\kbrdelim}{)}
\[
\kbordermatrix{
    & z_1 & z_2 & z_3 & z_4 & z_5 & w^{(1)}_{1,1} & w^{(1)}_{2,1}\\
    r_1 & 1 & 1 & 0 & 0 & 0 & 0 & 0 \\
    r_2 & 1 & 0 & 1 & 1 & 0 & 0 & 0 \\
    r_3 & 1 & 0 & 1 & 0 & 1 & 0 & 0 \\
    r_4 & 1 & 0 & 1 & 0 & 0 & 0 & 1 \\
    r_5 & 1 & 0 & 0 & 0 & 0 & 1 & 0
  },
\]
for tree 1 and 
\[
\kbordermatrix{
    & z_6 & z_7 & z_8 & z_9 & z_{10} & z_{11} & z_{12} & w^{(2)}_{2,1} & w^{(2)}_{3,1}\\
    r_1 & 1 & 1 & 1 & 0 & 0 & 0 & 0 & 0 & 0\\
    r_2 & 1 & 1 & 0 & 1 & 0 & 0 & 0 & 0 & 0\\
    r_3 & 1 & 0 & 0 & 0 & 1 & 1 & 0 & 0 & 0\\
    r_4 & 1 & 0 & 0 & 0 & 1 & 0 & 1 & 0 & 0\\
    r_5 & 1 & 0 & 0 & 0 & 1 & 0 & 0 & 0 & 1\\
    r_6 & 1 & 0 & 0 & 0 & 0 & 0 & 0 & 1 & 0
  },
\]
for tree 2. Through row operations they can be turned into network matrices, where there is exactly two non-zeros in each column, a 1 and a $-1$.
\[
\kbordermatrix{
    & z_1 & z_2 & z_3 & z_4 & z_5 & w^{(1)}_{1,1} & w^{(1)}_{2,1}\\
    -r_1 & -1 & -1 & 0 & 0 & 0 & 0 & 0 \\
    r_1-r_2 & 0 & 1 & -1 & -1 & 0 & 0 & 0 \\
    r_2-r_3 & 0 & 0 & 0 & 1 & -1 & 0 & 0 \\
    r_3-r_4 & 0 & 0 & 0 & 0 & 1 & 0 & -1 \\
    r_4-r_5 & 0 & 0 & 1 & 0 & 0 & -1 & 1 \\
    r_5 & 1 & 0 & 0 & 0 & 0 & 1 & 0
  },
\]
for tree 1 and 
\[
 \kbordermatrix{
    & z_6 & z_7 & z_8 & z_9 & z_{10} & z_{11} & z_{12} & w^{(2)}_{2,1} & w^{(2)}_{3,1}\\
    -r_1 & -1 & -1 & -1 & 0 & 0 & 0 & 0 & 0 & 0\\
    r_1-r_2 & 0 & 0 & 1 & -1 & 0 & 0 & 0 & 0 & 0\\
    r_2-r_3 & 0 & 1 & 0 & 1 & -1 & -1 & 0 & 0 & 0\\
    r_3-r_4 & 0 & 0 & 0 & 0 & 0 & 1 & -1 & 0 & 0\\
    r_4-r_5 & 0 & 0 & 0 & 0 & 0 & 0 & 1 & 0 & -1\\
    r_5-r_6 & 0 & 0 & 0 & 0 & 1 & 0 & 0 & -1 & 1\\
    r_6 & 1 & 0 & 0 & 0 & 0 & 0 & 0 & 1 & 0
  }
\]
for tree 2.
Note the above transformation to network matrices can always be done as long as the nodes are numbered in a pre-order fashion. 
Now we are ready to state the main theoretical result of this paper. 

\begin{theorem}
The linear program relaxation of \eqref{eq:IP3} has only integral optimal solutions.
\end{theorem}
\begin{proof}
Denote the equality constraints of \eqref{eq:IP3} with index set $J_1$. They can be divided into each tree. 
Each constraint matrix in $J_1$ associated with a tree can be turned into a network matrix according to Lemma \ref{lemma1}. Stacking these matrices leads to a larger network matrix. Denote the $w^{(t)}_{k,i}\leq w_{k,i}$ constraints with index set $J_2$. Consider the constraint matrix for $J_2$. Each $w^{(t)}_{k,i}$ only appears once in $J_2$, which means the column corresponding to $w^{(t)}_{k,i}$ has only one element equal to 1 and the rest equal to 0. If we arrange the constraints in $J_2$ such that for any given $k,i$ $w^{(t)}_{k,i}\leq w_{k,i}$ are put together for $t\in [T]$, the constraint matrix for $J_2$ has interval structure such that the non-zeros in each column appear consecutively. Finally, putting the network matrix from $J_1$ and the matrix from $J_2$ together. Assign $J_1$ and the odd rows of $J_2$ to the first partition $Q_1$ and assign the even rows of $J_2$ to the second partition $Q_2$. Note the upper bound constraints on the variables can be ignored as this is an minimization problem. We conclude that the constraint matrix of \eqref{eq:IP3} is totally unimodular according to Theorem 2.7, Part 3 of \cite{Nemhauser:1988:ICO:42805} with partition $Q_1$ and $Q_2$. By Proposition 2.1 and 2.2, Part 3 of \cite{Nemhauser:1988:ICO:42805} we can conclude the proof.
\end{proof}

We say a pruned tree of $\cT$ is \emph{optimal} for a given budget constraint if it has the lowest empirical error among all pruned trees of $\cT$ that satisfy the budget constrain.
\begin{corollary} \label{cor:1}
The linear program relaxation of \eqref{eq:IP3} produces an \emph{optimally pruned tree} for a given budget $B$.
\end{corollary}
\begin{proof}
Let the optimal value of \eqref{eq:IP3} be $f(\lambda)$. As $\lambda$ increases, a higher penalty is applied to the feature cost compared to the classification error; therefore, the optimal solution will have feature cost decreasing to 0 as a function of $\lambda$. Let $\lambda^*$ be such that the feature cost 
$\sum_{k=1}^{K}c_k(\frac{1}{N}\sum_{i=1}^{N}w^*_{k,i})=B$. Therefore, 
\begin{align*}
f(\lambda^*)&=\frac{1}{NT}\displaystyle  \sum_{t=1}^{T}\sum_{h\in \mathcal{N}^{(t)}} e^{(t)}_h z^{(t)*}_h +\lambda^* \sum_{k=1}^{K}c_k(\frac{1}{N}\sum_{i=1}^{N}w_{k,i}^*)\\
&=\frac{1}{NT}\displaystyle  \sum_{t=1}^{T}\sum_{h\in \mathcal{N}^{(t)}} e^{(t)}_h z^{(t)*}_h +\lambda^* B.
\end{align*}
On the other hand, consider \eqref{eq:IP3} with explicit budget constraint:
\begin{equation*}
\begin{array}{rll}
\displaystyle \min_{z^{(t)}_h,w^{(t)}_{k,i}\in Q} & \multicolumn{1}{l}{\frac{1}{NT}\displaystyle  \sum_{t=1}^{T}\sum_{h\in \mathcal{N}^{(t)}} e^{(t)}_h z^{(t)}_h } \\
\textrm{s.t.} & \sum_{k=1}^{K}c_k(\frac{1}{N}\sum_{i=1}^{N}w_{k,i}) \leq B
\end{array},
\tag{LP1}\label{eq:LP1}
\end{equation*}
where $Q$ denotes the constraint set of \eqref{eq:IP3}. Let $\text{opt}$ be the optimal value of \eqref{eq:LP1}. Then we have
\begin{align*}
\text{opt}&=\min_{z^{(t)}_h,w_{k,i}\in Q} \max_{\lambda \geq 0} & ( \frac{1}{NT}\displaystyle  \sum_{t=1}^{T}\sum_{h\in \mathcal{N}} e^{(t)}_h z^{(t)}_h  \\
& &+\lambda (\sum_{k=1}^{K}c_k\frac{1}{N}\sum_{i=1}^{N}w_{k,i}-B) )\\
& = \max_{\lambda \geq 0} \min_{z^{(t)}_h,w_{k,i}\in Q}  & ( \frac{1}{NT}\displaystyle  \sum_{t=1}^{T}\sum_{h\in \mathcal{N}} e^{(t)}_h z^{(t)}_h  \\
& &+\lambda (\sum_{k=1}^{K}c_k\frac{1}{N}\sum_{i=1}^{N}w_{k,i}-B) ).
\end{align*}
By the definition of $f(\lambda)$ we have
\begin{align*}
\text{opt} & = \max_{\lambda \geq 0} f(\lambda)-\lambda B  \\
& \geq f(\lambda^*)-\lambda^* B  \\
& =  \frac{1}{NT}\displaystyle  \sum_{t=1}^{T}\sum_{h\in \mathcal{N}^{(t)}}  e^{(t)}_h z^{(t)*}_h  +\lambda^* B -\lambda^* B \\
& = \frac{1}{NT}\displaystyle  \sum_{t=1}^{T}\sum_{h\in \mathcal{N}^{(t)}}  e^{(t)}_h z^{(t)*}_h 
\end{align*}
Thus we obtain the desired inequality.
\end{proof}

\paragraph{Complexity:} The number of $z^{(t)}_h$ variables is at most $T \times |\cT_{\text{max}|}$, where $|\cT_{\text{max}|}$ is the maximum number of nodes in a tree.
The number of $w^{(t)}_{k,i}$ variables is at most $N \times T \times K_{\text{max}}$, where $K_{\text{max}}$ is the maximum number of features an example uses in a tree. The number of $w_{k,i}$ variables is at most $N \times \min\{T \times K_{\text{max}},K\}$. In total there are $T \times |\cT_{\text{max}|}+N \times T \times K_{\text{max}} + N \times \min\{T \times K_{\text{max}},K\}$ variables. The number of $ z^{(t)}_h+ \sum_{u\in p(h)} z^{(t)}_u=1$ constraints is at most $T\times |\tilde{\cT}_{\text{max}}|$, where $|\tilde{\cT}_{\text{max}}|$ is the maximum number of leaf nodes in any tree. The number of $w^{(t)}_{k,i}+ z^{(t)}_{u_{t,k,i}}+\sum_{h\in p(u_{t,k,i})} z^{(t)}_h=1$ constraints is at most $N \times T \times K_{\text{max}}$. The number of $w^{(t)}_{k,i} \leq w_{k,i}$ constraints is again at most $N \times T \times K_{\text{max}}$. In total there are at most $T\times |\tilde{\cT}_{\text{max}}| +2 \times N \times T \times K_{\text{max}}$ constraints besides the positivity constraints on all variables.

\section{Parallel Ensemble Pruning}
In this section we further explore the special structure of \eqref{eq:IP3} and show that it admits a Dantzig-Wolfe decomposition that can be massively parallelized. The key observation is that pruning each tree is a shortest-path problem on directed graphs that can be efficiently solved ($O(|\tilde{\cT}|^2)$). First, group the variables $\{z^{(t)}_h, w^{(t)}_{k,i}, \forall h \in \tilde{\mathcal{T}}_t, \forall k\in K_{t,i},\forall i\in S\}$ into a vector $\theta^{(t)}$ for each tree $\cT_t, t=1,\dots,T$.
Let $P_t$ denote the feasible set corresponding to the first 4 sets of (LP-relaxed) constraints in \eqref{eq:IP3} for tree $\cT_t$:
\begin{align*}
P^{(t)}&=  \{\theta^{(t)} = (z^{(t)}_h, w^{(t)}_{k,i}) |
z^{(t)}_h+ \sum_{u\in p(h)} z^{(t)}_u=1 , \forall h \in \tilde{\mathcal{T}}_t,\\
& w^{(t)}_{k,i}+ z^{(t)}_{u_{t,k,i}}+\sum_{h\in p(u_{t,k,i})} z^{(t)}_h=1 , 
 \forall k\in K_{t,i},\forall i\in S, \\
& z^{(t)}_h \geq 0, w^{(t)}_{k,i}\geq 0, \forall h\in \cT_t, \forall k\in K_{t,i}, \forall i\in [N]
 \}.
\end{align*}
Thus, the LP relaxation of \eqref{eq:IP3} can be re-written as 
\begin{equation*}
\begin{array}{rlr}
\displaystyle \min_{z^{(t)}_h,w^{(t)}_{k,i}} &  \multicolumn{2}{l}{\frac{1}{NT}\displaystyle  \sum_{t=1}^{T}\sum_{h\in \cT_t} e^{(t)}_h z^{(t)}_h +\lambda \sum_{k=1}^{K}c_k(\frac{1}{N}\sum_{i=1}^{N}w_{k,i})} \\
\textrm{s.t.} & \theta^{(t)} \in P^{(t)} & \forall t\in [T], \\
& w^{(t)}_{k,i} \leq w_{k,i} & \forall k\in [K], \forall i\in S \forall t\in [T],\\
& w_{k,i}\geq 0 & \forall k\in [K], \forall i\in S.
\end{array}
\tag{LP2}\label{eq:LP2}
\end{equation*}
Let $\hat{\theta}^{(t)}_i$,$i=1,\dots,I_t$ be the extreme points of $P^{(t)}$. Any point in $P^{(t)}$ can be written as a convex combination of these extreme points: $\theta^{(t)}=\sum_{j=1}^{I_t} \alpha^{(t)}_j \hat{\theta}^{(t)}_j, \sum_{j=1}^{I_t} \alpha^{(t)}_j =1, \alpha^{(t)}_j\geq 0$. Thus we re-write \eqref{eq:LP2} in terms of the extreme points of $P^{(t)}$:
\begin{equation*}
\resizebox{1\hsize}{!}{$\begin{array}{rlr}
\displaystyle \min_{\alpha^{(t)}_j,w_{k,i}} &  \multicolumn{2}{l}{\frac{1}{NT}\displaystyle  \sum_{t=1}^{T}\sum_{h\in \cT_t} \sum_{j=1}^{I_t} e^{(t)}_h \alpha^{(t)}_j \hat{z}^{(t)}_{h,j} +\lambda \sum_{k=1}^{K}c_k(\frac{1}{N}\sum_{i=1}^{N}w_{k,i})} \\
\textrm{s.t.} & \sum_{j=1}^{I_t} \alpha^{(t)}_j \hat{w}^{(t)}_{k,i,j} \leq w_{k,i} & \forall i\in S,\forall k\in [K], \forall t\in [T], \\
& \sum_{j=1}^{I_t} \alpha^{(t)}_j =1, \alpha^{(t)}_j\geq 0, & \forall t\in [T],\\
& w_{k,i}\geq 0 & \forall k\in [K], \forall i\in S,
\end{array}$}
\tag{LP3}\label{eq:LP3}
\end{equation*}
where $\hat{z}^{(t)}_{h,j}$ is the $j$th extreme point value of the node $h$ on tree $\cT_t$ and $\hat{w}^{(t)}_{k,i,j}$ is the $j$th extreme point value of $w^{(t)}_{k,i}$. 
In a more compact form, we can write \eqref{eq:LP3} as
\begin{equation}\resizebox{1\hsize}{!}{$
\begin{aligned}
 \min_{\alpha^{(t)}_i, w_{k,i}} & \frac{1}{NT}\sum_{t=1}^{T}\sum_{j=1}^{I_t} \alpha^{(t)}_j  \mathbf{c}'_t \hat{\theta}^{(t)}_{j}+ \lambda \sum_{k=1}^{K} c_k (\frac{1}{N}\sum_{i=1}^N w_{k,i}) \\
 \text{s.t.} & \sum_{j=1}^{I_1} \alpha^{(1)}_j \left( 
\begin{array}{c}
D^{(1)}\hat{\theta}^{(1)}_j \\
1\\
0\\
\vdots \\
0
\end{array}
\right) + \dots +
\sum_{j=1}^{I_T} \alpha^{(T)}_j \left( 
\begin{array}{c}
D^{(T)}\hat{\theta}^{(T)}_j \\
0\\
0\\
\vdots \\
1
\end{array}
\right)  \\
& 
+\left(
\begin{array}{c}
D_w\mathbf{w} \\
0\\
0\\
\vdots \\
0
\end{array}
\right) +
\left(
\begin{array}{c}
D_s\mathbf{s} \\
0\\
0\\
\vdots \\
0
\end{array}
\right) =
\left(
\begin{array}{c}
\mathbf{0} \\
1\\
1\\
\vdots \\
1
\end{array}
\right),&  \\
 & \alpha^{(t)}_j\geq 0, \forall t \in [T], \forall j \in [I_t] \\
 & w_{k,i} \geq 0, \forall k\in [K],\forall i\in [N], 
\end{aligned}$}\tag{LP4}\label{eq:LP4}
\end{equation}
where $D^{(t)}$ is the constant matrix selecting the $\hat{w}^{(t)}_{k,i,j}$ components of $\hat{\theta}^{(t)}_j$; $\mathbf{w}$ is the vector of $w_{k,i}$'s and $\mathbf{s}$ is the vector of slack variables $s^{(t)}_{k,i}$'s. The number of equality constraints is at most $N \times T \times K_{\text{max}}+T$, much less than the number of constraints in \eqref{eq:IP3}. However, the number of variables can be huge. 

The Danzig-Wolfe algorithm works as follows. Start with a feasible basis $B$ of \eqref{eq:LP4} and a dual vector $\mathbf{p}'=(\mathbf{q}',r_1,\dots,r_T)=\mathbf{c}'_B B^{-1}$, where $\mathbf{q}$ corresponds to the constraints involving $w$'s and $r_1,\dots,r_T$ corresponds to the convexity constraints of the $\alpha$'s. 
For each tree $t$, solve the sub-problem 
\begin{equation}
\begin{aligned}
OPT_t= & \min_{\bx} & (\mathbf{c}'_t-\mathbf{q}'D^{(t)})\bx \\
& \text{subject to } & \bx^{(t)} \in P^{(t)}.
\end{aligned}\tag{SUB1}\label{opt:DW-sub1}
\end{equation} 
If $OPT_t<r_t$, and the extreme point $\hat{\bx}^{(t)}_i$ is optimal to the above sub-problem then it is easy to check that the reduced cost for the variable $\alpha^{(t)}_j$ is less than 0:
\begin{align*}
&\mathbf{c}'_t \hat{\theta}^{(t)}_j-
\left[
\begin{array}{cccc}
\mathbf{q'} & r_1 & \dots & r_T 
\end{array} \right]
\left[
\begin{array}{c}
D^{(t)}\hat{\theta}^{(t)}_j \\
0\\
\vdots \\
1 \\
\vdots \\
0
\end{array} \right] \\
&=
(\mathbf{c}'_t-\mathbf{q}'D^{(t)})\hat{\theta}^{(t)}_j-r_t<0.
\end{align*}
Therefore, generate column 
$$
(D^{(t)}\hat{\theta}^{(t)}_i, 0,\dots,1,\dots,0)^T 
$$
and bring it into basis. Note due to the network matrix structure (Lemma \ref{lemma1}), these subproblems can be solved very efficiently.
Similarly, check the reduced costs for all $w_{k,i}$'s and $s^{(t)}_{k,i}$'s, and if any of them are negative, generate the corresponding columns and bring them into the basis. For the above decomposition, the main computational burden of pruning individual trees can be distributed to separate computional nodes that communicate to adjust for shared features. This can lead to dramatic efficiency improvement when the number of trees in the ensemble becomes large. Efficient implementation of the Dantzig-Wolfe decomposition has been shown to yield significant speedup through parallelism \cite{tebboth2001computational}. 

\begin{table*}[t]
\centering
\scalebox{0.83}{
\begin{tabular}{|c|c|c|c|c|c|c|}
\hline
    &  & no pruning & ens.pru.low            & ind.pru.low & ens.pru.high & ind.pru.high\\ \hline
\multirow{2}{*}{MiniB} & cost                         & 37.0671$\scriptstyle \pm 0.3108$          & (68.24)25.2960$\scriptstyle \pm 0.3157$                      & (95.02)35.2219$\scriptstyle \pm 0.3667$           & (43.17)16.0018$\scriptstyle \pm 0.2498$       &  (55.19)20.4584$\scriptstyle \pm 0.1270$  \\ 
                           & error                      & 0.0725$\scriptstyle \pm 0.0004$          & 0.0724$\scriptstyle \pm 0.0005$                      & 0.0727$\scriptstyle \pm 0.0004$           & 0.0766$\scriptstyle \pm 0.0004$      &  0.0766$\scriptstyle \pm 0.0008$    \\  \hline
\multirow{2}{*}{Forest}    & cost    & 13.9005$\scriptstyle \pm 0.0498$          & (88.10)12.2463$\scriptstyle \pm 0.0834$ & (93.24)12.9604$\scriptstyle \pm 0.1004$           & (65.16)9.0577$\scriptstyle \pm 0.6481$       &  (78.82)10.9565$\scriptstyle \pm 0.0729$  \\ 
                           & error & 0.1122$\scriptstyle \pm 0.0009$          & 0.1135$\scriptstyle \pm 0.0010$ & 0.1137$\scriptstyle \pm 0.0010$           & 0.1220$\scriptstyle \pm 0.0025$       &  0.1228$\scriptstyle \pm 0.0010$  \\ \hline
\multirow{2}{*}{Cifar}     & cost                         & 186.5456$\scriptstyle \pm 1.3180$            & (92.40)172.3720$\scriptstyle \pm 1.8741$                      & (93.02)173.5255$\scriptstyle \pm 1.4516$           & (75.39)140.6308$\scriptstyle \pm 2.5059$      &  (77.89)145.2933$\scriptstyle \pm 2.4797$   \\ 
                           & error                      & 0.3152$\scriptstyle \pm 0.0031$           & 0.3165$\scriptstyle \pm 0.0021$                      & 0.3158$\scriptstyle \pm 0.0024$           & 0.3227$\scriptstyle \pm 0.0026$      &  0.3236$\scriptstyle \pm 0.0016$   \\  \hline
\multirow{2}{*}{Sonar}     & cost                         & 49.9715 $\scriptstyle \pm 1.1103$             & (45.20)22.5860$\scriptstyle \pm 3.9528$                      & (74.31)37.1355$\scriptstyle \pm 2.0425$           & (16.48)8.2349$\scriptstyle \pm 1.7930$       &  (28.11)14.0479$\scriptstyle \pm 1.6909$  \\ 
                           & error                      & 0.1539$\scriptstyle \pm 0.0641$           & 0.1838$\scriptstyle \pm 0.0722$                      & 0.1890$\scriptstyle \pm 0.0691$           & 0.2121$\scriptstyle \pm 0.0668$      & 0.2139$\scriptstyle \pm 0.0676$    \\  \hline
\multirow{2}{*}{Heart}     & cost                         & 12.1670$\scriptstyle \pm 0.2341$            & (73.26)8.9133$ \scriptstyle \pm 0.7524$                      & (96.20)11.7052$\scriptstyle \pm 0.2742$           & (47.75)5.8094$\scriptstyle \pm 2.5589$     & (75.86)9.2301$\scriptstyle \pm 0.7036$     \\ 
                           & error                      & 0.1721$\scriptstyle \pm 0.0756$           & 0.1711$\scriptstyle \pm 0.0727$                      & 0.1719$\scriptstyle \pm 0.0680$           & 0.1977$\scriptstyle \pm 0.0807$     &  0.1973$\scriptstyle \pm 0.0724$    \\  \hline
\end{tabular}
}
\caption[]{Comparison of no pruning, ensemble pruning and individual pruning in terms of average feature costs and test error. Two different error levels for both ensemble and individual pruning methods are reported. Cost of the pruned trees are also reported as percentages of the cost of the unpruned trees in parenthesis.}
   	\label{table:RF}
\end{table*}

\section{Experiments}\label{sec:experiments}
We test our pruning algorithm on a number of benchmark datasets to show its advantage. Our pruning takes the ensembles from BudgetRF algorithm \cite{icml2015_nan15} as input. The datasets are CIFAR \citealp{CIFAR10}, MiniBooNE, Forest Covertype, Sonar and Heart \cite{UCI_repository}.

Note the cost of each feature is 1 uniformly in all datasets and therefore cost is equivalent to the average number of unique features used for each example. For each dataset, we present the average cost and average error on test data in Table \ref{table:RF}. As a baseline, we provide the performance of BudgetRF without pruning in the third column of Table \ref{table:RF}. The results of our proposed ensemble pruning methods are in columns 4 and 6 under the title ``ens.pru.".
We also compare to the same pruning algorithm that we propose but applied to individual trees separately rather than the entire ensemble. The results are given in column 5 and 7 of Table \ref{table:RF} under the title ``ind.pru.". Intuitively, pruning as an ensemble exploits the interdependencies among trees, potentially leading to better accuracy-cost trade-offs compared to pruning individual trees separately. We present pruning results at two different error levels: low and high. A low error level corresponds to little pruning with an implicitly larger average cost, while a high error level corresponds to pruning away much of the original trees, reducing the average cost at the expense of reduced classification performance.

For ease of comparison, we match the error levels for both ensemble and individual pruning methods and focus on the difference in cost. We also compute the cost as a percentage of the cost of the unpruned ensemble, shown in parenthesis in Table \ref{table:RF}.

In MiniBooNE, Forest and CIFAR datasets, we run BudgetRF to obtain an ensemble of 40 trees following the given training/validation/test data splits \cite{icml2015_nan15}. We report the mean and standard deviations based on 10 repeated runs. We observe that ensemble pruning reduces cost of the BudgetRF ensembles significantly while keeping the same level of test error. For example, the unpruned ensemble on MiniBooNE uses about 37 features on an average test example with an average test error of 0.0725; our ensemble pruning method reduces the average number of features to about 25, about 68\% of the unpruned cost, with test error 0.0724. Further reduction of the cost to $43\%$ of the original budget maintains approximately the same level of accuracy.

In Sonar and Heart datasets, we run BudgetRF to obtain an ensemble of 90 trees. Because of the small sizes, we perform 10-fold cross validation to obtain training/test splits and report the mean as well as standard deviation of test cost and error over 100 repeated runs. Again we observe the effectiveness of our pruning algorithm. For example in Heart the ensemble pruning uses 73\% of the unpruned cost without losing accuracy.

We observe that ensemble pruning always performs better than individual pruning: fixing the error levels, Table \ref{table:RF} shows that ensemble pruning always incurs less feature cost than individual pruning. The advantage is quite significant in most of the datasets. This is expected because pruning individual trees does not exploit the inter-dependencies among trees.

\section{Conclusion} 
We propose a novel ensemble pruning formulation with feature costs involving a 0-1 integer program. We prove that the linear program relaxation produces the optimal solution to the original integer program. This allows us to use efficient convex optimization tools to obtain the optimally pruned ensemble for any given budget. Our pruning formulation is general - it can take any ensemble of decision trees as input. As the pruning formulation explicitly account for feature sharing across trees together with accuracy/cost trade-off, it is able to significantly reduce feature cost by pruning subtrees that introduce more loss in terms of feature cost than benefit in terms of prediction accuracy gain. Empirically we see that our pruning algorithm indeed significantly improves the performance of the state of the art ensemble method BudgetRF.

\bibliography{cost_sensitive_bib}
\bibliographystyle{icml2015}

\end{document}